\documentclass[accepted]{article}

\usepackage{aistats2025}

\usepackage{amsthm}
\usepackage{amsmath}
\usepackage{amssymb}
\usepackage{graphicx}
\usepackage{booktabs}
\usepackage{multirow}
\usepackage[ruled,vlined]{algorithm2e}

\usepackage[colorlinks=true, urlcolor=blue, linkcolor=black, citecolor=black]{hyperref}
\usepackage[table]{xcolor}

\usepackage{caption}

\usepackage{subcaption}
\usepackage{amsthm}
\usepackage{listings}

\usepackage{array}
\newcolumntype{C}[1]{>{\centering\arraybackslash}p{#1}}

\newcolumntype{L}[1]{>{\raggedright\arraybackslash}m{#1}} 

\newtheorem{definition}{Definition}[section]
\newtheorem{theorem}{Theorem}[section]
\newtheorem{corollary}{Corollary}[theorem]

\newtheorem{assumption}{Assumption}[section]

\definecolor{lightgray}{gray}{0.95}

\usepackage[round]{natbib}

\begin{document}
\runningtitle{Algorithmic Accountability in Small Data}
\runningauthor{Jarren Briscoe, Garrett Kepler, Daryl Deford, Assefaw Gebremedhin}

\newcommand{\BibTeX}{B\kern-.05em{\sc i\kern-.025em b}\kern-.08em\TeX}

\twocolumn[

\aistatstitle{Algorithmic Accountability in Small Data:\\Sample-Size-Induced Bias Within Classification Metrics}

\aistatsauthor{Jarren Briscoe$^{1}$ \And Garrett Kepler$^{2}$ \And Daryl Deford$^{2}$ \And Assefaw Gebremedhin$^{1}$}

\aistatsaddress{$^{1}$School of Electrical Engineering \& Computer Science, Washington State University \\ $^{2}$Department of Mathematics and Statistics, Washington State University}
]

\begin{abstract}
Evaluating machine learning models is crucial not only for determining their technical accuracy but also for assessing their potential societal implications. While the potential for low-sample-size bias in algorithms is well known, we demonstrate the significance of sample-size bias induced by combinatorics in classification metrics. This revelation challenges the efficacy of these metrics in assessing bias with high resolution, especially when comparing groups of disparate sizes, which frequently arise in social applications. We provide analyses of the bias that appears in several commonly applied metrics and propose a model-agnostic assessment and correction technique. Additionally, we analyze counts of undefined cases in metric calculations, which can lead to misleading evaluations if improperly handled. This work illuminates the previously unrecognized challenge of combinatorics and probability in standard evaluation practices and thereby advances approaches for performing fair and trustworthy classification methods.
\end{abstract}


\section{INTRODUCTION}
Classification metrics derived from confusion matrices are fundamental tools in evaluating machine learning models, particularly in binary classification tasks. Metrics such as accuracy, precision, recall, and Matthews Correlation Coefficient (MCC) provide critical insights into model performance. However, an often-overlooked issue is the impact of sample size on these metrics, especially when comparing performance across different groups or datasets of varying sizes.

Small sample sizes introduce jaggedness (significant variability with discrete ``jumps'' in the metrics' score distributions) due to the discrete and combinatorial nature of confusion matrices. This variability can cause misleading interpretations of a model's performance and fairness. For instance, minor changes in the sample size can cause disproportionate shifts in metric score distributions, potentially exaggerating or obscuring biases among groups.

Moreover, certain metrics become undefined under specific conditions, such as divisions by zero. These undefined cases, or ``holes'', in the metric space further challenge the reliability of performance assessments, especially in groups with limited data.

Despite the widespread use of classification metrics, there has been limited attention to the systematic biases and inconsistencies induced by sample size variations. Existing literature often assumes large sample sizes or overlooks the combinatorial complexities that small samples introduce, leaving a gap in understanding how to accurately assess model performance.

We address these challenges by providing a comprehensive analysis of sample-size-induced biases in confusion-matrix metrics. Our contributions include:

\begin{itemize} \item \textbf{Demonstrating Metric Variability}: We present both theoretical and empirical evidence of the jaggedness and variability in classification metrics caused by small sample sizes. By illustrating how these effects distort performance evaluations, we underscore the need for careful interpretation of metrics in small-sample scenarios.

\item \textbf{Quantifying Undefined Cases}: We systematically count the number of ``holes''—situations where metrics are undefined—in 21 metrics. This quantification reveals the extent to which certain metrics may be unreliable or misleading under specific sample configurations.

\item \textbf{Metric Alignment Trial for Checking Homogeneity (MATCH) Test}: We introduce a statistical test that assesses the significance of a group's metric score by comparing it to the distribution of scores from a reference group. This approach allows us to determine whether observed differences are due to genuine performance disparities or merely artifacts of sample size variability.

\item \textbf{Cross-Prior Smoothing (CPS)}: We propose a smoothing technique that incorporates prior information from other groups to enhance metric reliability. By adjusting confusion matrix counts with cross-group priors, CPS improves the stability and reduces the error rates of metric estimates.

\item \textbf{Open-Source Implementation}: We provide an open-source repository of our code and experiments to facilitate reproducibility and encourage further research in this area.\footnote{\href{https://github.com/jarrenbr/Algorithmic-Accountability-in-Small-Data}{https://github.com/jarrenbr/Algorithmic-Accountability-in-Small-Data}}

\end{itemize}

Our work illuminates a critical issue with classification metrics and offers practical solutions to assess and improve their reliability. By addressing the biases introduced by sample size, we aim to enhance the fairness and accuracy of model evaluations, particularly in applications where group comparisons are essential.

\section{RELATED WORK}
As machine learning models are increasingly used in high-stakes domains, ensuring fairness across groups is critical. However, the combinatorial nature of classification metrics remains underexplored. This study examines how these metrics evolve with changing sample sizes, revealing potential biases that may exacerbate or hide disparities among groups. Understanding these distributional shifts can help develop more equitable AI systems, particularly in settings where sample sizes differ significantly. We provide metric definitions in Section~\ref{sec:metricDef}.

Confusion matrices are fundamental tools for assessing classification performance, providing the basis for scalar metrics like accuracy and MCC. While these metrics are widely used, they can exhibit biases related to sample size. As such, this bias affects common fairness metrics, including disparate impact \citep{diOriginal}, equalized odds \citep{equalizedOddsOriginal}, and predictive parity \citep{amazonFairness}. By highlighting this issue, we aim to enhance the robustness of metric comparisons across datasets with varying sample sizes.

In social data analyses, accurate measurement of fairness is crucial, particularly in legal contexts where metrics like disparate impact assess potential discrimination. For instance, in employment or housing discrimination cases \citep{SPP_DI,LPR_DI} and recidivism prediction \citep{recidivismFairness,Moore2023Pretrial}, datasets often feature small or imbalanced sample sizes. Recognizing how distributional shifts in metrics affect evaluations is vital for ensuring fair outcomes. The recently introduced, legally based metric, Objective Fairness Index \citep{briscoe2024facets} is also susceptible to such shifts.

For different reasons, models themselves can create similar shifts in predictions over groups of data (instead of sample size). Recent work by \citet{feng2024model} frames this as a changepoint detection problem and detect if models are properly calibrated for all groups. However, they do not consider classification metrics nor their combinatorics. Our approach complements their method by specifically addressing biases in these classification metrics, which are critical for fairness evaluations.

Statistical methods designed for small sample sizes are relevant to our approach. Techniques like Laplacian smoothing, commonly used in natural language processing to handle unobserved events \citep{NLP_Book}, can be adapted to confusion matrices where certain classes may not be represented. Discussions around the use of non-informative or weakly informative priors in Bayesian estimation for small samples \citep{BISSS, WIP} inform our methodology in developing robust methods under data limitations.

\citet{SOKOLOVA2009427} provides a comprehensive analysis of various performance measures used in classification tasks, focusing on their invariance properties with respect to changes in the confusion matrix. Their work introduces the concept of ``measure invariance'', which refers to a metric's ability to maintain consistent evaluations despite transformations in the underlying data distribution.

Additionally, \citet{goutte2005probabilistic} introduces a probabilistic framework that interprets precision, recall, and the F$_1$ score by using a symmetric beta distribution and Monte Carlo sampling techniques. This approach moves beyond simple point estimates by estimating the likelihood that one algorithm would outperform another based on these scores. Part of our work extends this idea by generalizing precision and recall to a broader set of metrics, grouped as \emph{Joint Ratio Metrics}.

The issue of metric unreliability in small sample sizes is often overlooked, with practitioners sometimes attributing it to the need for more data. \citet{chicco2020advantages} specifically addresses this issue along with imbalanced data by comparing MCC against the more widely used accuracy and F$_1$-score. They argue that MCC provides a more reliable assessment of classifier performance in binary tasks, particularly when there is a class imbalance because MCC incorporates all four confusion matrix categories. In contrast, accuracy and F$_1$ score can give misleadingly high values in imbalanced datasets, failing to capture poor performance in one class.

In another work, \citet{rudner2024mind} adjusts model parameters with Group-Aware Priors to improve model robustness under subpopulation shifts. \citet{aghbalou2024sharp} examines imbalanced classification and derives sharp error bounds under class imbalance. Their focus is on improving classification performance when one class is underrepresented, offering new theoretical insights on error rates when class probabilities approach zero.

Our work builds on these foundations by examining the effect of small sample sizes on a broader range of classification metrics. We demonstrate that metrics like accuracy, precision, recall, and even MCC not only exhibit increased variability but also have discrete jumps and falls (``jaggedness'') as the sample size scales, potentially leading to misleading evaluations. For instance, as \citeauthor{chicco2020advantages} present, while MCC is robust in imbalanced scenarios, it can still fluctuate in extreme cases of imbalance or small samples, an issue our Cross-Prior Smoothing (CPS) technique aims to mitigate by using prior information from reference groups. Additionally, our introduction of the MATCH Test allows for a rigorous assessment of whether observed metric scores in small groups are meaningful or merely artifacts of sampling variability.

By addressing these gaps, we contribute to the growing literature on improving the robustness and fairness of classification inference. Our approach offers both theoretical insights and practical tools to ensure more reliable evaluations across diverse groups, particularly when sample sizes are limited or imbalanced.

\section{DISTRIBUTION SHIFTS}
\begin{figure}[t]
    \centering
    \includegraphics[width=\linewidth]{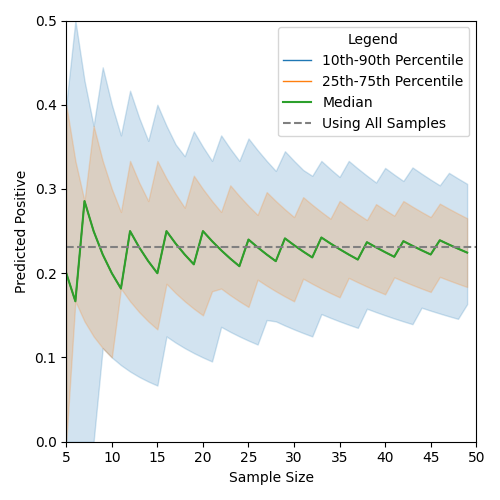}
    \caption{Variability of Positive Predictive Rate for Wealth Classification Among Multiracial Individuals using the \hyperref[sec:folktablesIncomeData]{Folktables' Income Dataset}: A Monte Carlo Simulation Study Based on Sample Size.}
    \label{fig:pprMultiracial}
\end{figure}

The variability and jaggedness in the predicted positive rates, as illustrated in Figure~\ref{fig:pprMultiracial}, is a direct consequence of both the probabilistic nature of small sample sizes and the combinatorial properties of confusion matrices with scaling sample sizes. This section explains the underlying phenomena contributing to these effects through a probabilistic lens.

Let us begin by considering the structure of the confusion matrix CM in Definition~\ref{def:cm}.

\begin{definition}\label{def:cm}
    The binary confusion matrix, CM, is a $2\times2$ matrix of the non-negative counts of true positives (TP), false negatives (FN), false positives (FP), and true negatives (TN). Each of these counts is referred to as a \emph{cell} $c$.
    \begin{align}
    \mathrm{CM} \triangleq \begin{bmatrix} TP & FN \\ FP & TN \end{bmatrix}
    \end{align}
\end{definition}
The total number of samples $n$ is partitioned among these four categories, such that $n$=TP+FN+FP+TN. For a given sample size $n$, we denote the set of all possible confusion matrices as $\mathcal{M}(n)$.

The narrowing of the distribution observed in Figure~\ref{fig:pprMultiracial} can be attributed to the combinatorial expansion of $\mathcal{M}(n)$ with increasing $n$. For small sample sizes, the number of possible confusion matrices is limited, causing each configuration (i.e., unique confusion matrix) to represent a significant fraction of the total space---including those yielding extreme metric values. As $n$ grows, the number of possible confusion matrices increases rapidly, and $\mathcal{M}(n)$ becomes densely populated. Specifically, the cardinality of $\mathcal{M}(n)$ is given by $\vert \mathcal{M}(n) \vert = \binom{n+3}{3}$ using the stars and bars method (Theorem~\ref{thm:totalOrdinalCount}). Hence, the number of possible confusion matrices grows cubically with $n$.

These configurations are influenced by the underlying probability distribution of classification outcomes, denoted as $(p_{\mathrm{TP}}, p_{\mathrm{FN}}, p_{\mathrm{FP}}, p_{\mathrm{TN}})$. Using $k_i$ as the integer count of $i$, we use the multinomial distribution definition to define the probability of any CM configuration:
\begin{align}\label{eq:multinomialPmf}
    P(CM) = \frac{n!}{k_\mathrm{TP}! \ k_\mathrm{FP}! \ k_\mathrm{FN}! \ k_\mathrm{TN}!} \ p_{\mathrm{TP}}^{k_\mathrm{TP}} \ p_{\mathrm{FP}}^{k_\mathrm{FP}} \ p_{\mathrm{FN}}^{k_\mathrm{FN}} \ p_{\mathrm{TN}}^{k_\mathrm{TN}}.
\end{align}
\begin{assumption}
    Using the multinomial distribution assumes that 1) the data is sampled i.i.d., and 2) that the model's classifications are independent.
\end{assumption}

Due to the factorial terms in the multinomial coefficient from Equation~\ref{eq:multinomialPmf}, even a single-unit increase in $n$ can significantly disrupt the probability mass distribution for small sample sizes, as depicted in Figure~\ref{fig:pprMultiracial}.

The distribution shifts in Figure~\ref{fig:metricCdfs} are not a probabilistic artifact as \textit{all} possible configurations are considered. As $n$ increases, the discrete space of possible confusion matrices becomes more finely granulated, and the number of configurations corresponding to certain metric scores increases at different rates.

\begin{figure}[t]\centering\includegraphics[width=\linewidth]{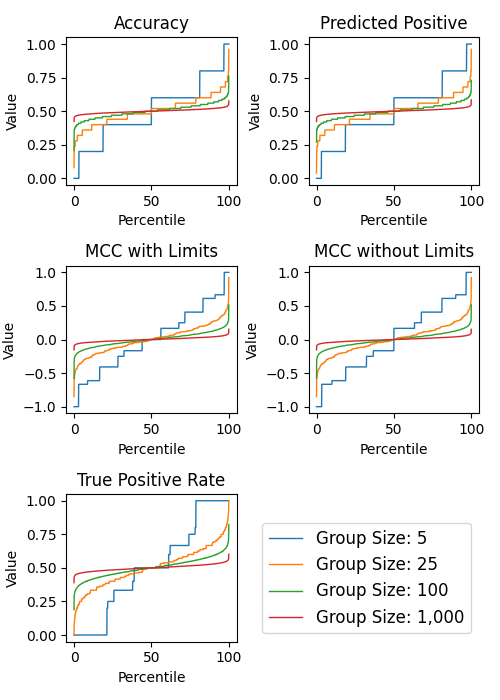}
    \caption{Distribution Shift in Classification: Analyzing Metric Sensitivity to Sample Size with Empirical Cumulative Distribution Functions.}
    \label{fig:metricCdfs}
\end{figure}

As the sample size increases, the probability masses of metric scores tend to converge to their expected values. In Figure~\ref{fig:pprMultiracial}, the expected value of the predicted positive rate converges to $1365/5956 \approx 0.23$. For any metric $M$, the expected value is $\mathbb{E}[M(p_{\mathrm{TP}}, p_{\mathrm{FN}}, p_{\mathrm{FP}}, p_{\mathrm{TN}})]$.

For example, Figure~\ref{fig:metricCdfs} uses all configurations, so all probabilities are 0.25. We see that accuracy and predicted positive converge to (0.25+0.25)/1=0.5; MCC (with or without limits) converges to $\frac{0.25\cdot 0.25-0.25\cdot0.25}{\sqrt{(0.25+0.25)\cdot(0.25+0.25)\cdot(0.25+0.25)\cdot(0.25+0.25)}}=0$; and TPR goes to 0.25/(0.25+0.25)=0.5.

By understanding 
discrete distribution shifts, we highlight the importance of considering sample size effects when interpreting classification metrics, particularly 
where small groups are compared. Ignoring these effects can lead to misleading conclusions about 
performance and fairness across different populations.

\section{UNDEFINED CASES}\label{sec:edgeCases}
            
\begin{table*}[h]
    \centering
    \caption{\label{tab:holeCounts}Undefined Cases Across Metrics Compared to Total Configuration Counts for all $n \geq 3$.}
    \rowcolors{2}{lightgray}{white}
    \renewcommand{\arraystretch}{1.25}
    \begin{tabular}{l c r}
        \toprule
        \textbf{Metric} & \textbf{Asymptotic Growth} & \multicolumn{1}{c}{\textbf{Undefined Case Count}} \\
        \midrule
        Binomial Metrics (ACC, PREV, $\dots$) & $\Theta(1)$ & 0\\
        Marginal Benefit & $\Theta(1)$ & 0\\
        Objective Fairness Index & $\Theta(1)$ & 0\\
        Simplified F$_1$ Score & $\Theta(1)$ & 1\\
        Joint Ratio Metrics (TPR, FPR, $\dots$) & $\Theta(n)$ & $n + 1$\\
        Matthews Correlation Coefficient & $\Theta(n)$ & $4n$\\
        \addlinespace
        Prevalence Threshold & $\Omega(n)$ \textbf{\&} $\mathcal{O}(n \log\log n)$ & $\geq 2n+2 \text{ and } < e^\gamma n\log \log n + \frac{0.6483n}{\log \log n}$\\
        \addlinespace
        Treatment Equality & $\Theta(n^2)$ & $\binom{n_1+2}{2}+\binom{n_2+2}{2} - 1$\\
        Original F$_1$ Score & $\Theta(n^2)$ & $\binom{n+2}{2}$\\
        \midrule
        Unique Confusion-Matrices Count & $\Theta(n^3)$ & Count of All Possible CMs: $\binom{n+3}{3}$\\
        \bottomrule
    \end{tabular}
\end{table*}

Classification metrics derived from confusion matrices can exhibit undefined behaviors, or ``holes", under certain conditions—particularly when denominators in their formulas become zero. In Table~\ref{tab:holeCounts}, we find that for all considered metrics, the asymptotic growth of undefined cases is less than the cubic growth of total possible configurations. We provide the proofs in Section~\ref{sec:holes}.

We generalize metrics of the form $M = \frac{c_i+c_j}{n}$ as \emph{Binomial Metrics}, which include common measures like accuracy, prevalence, and predicted positive rate. Furthermore, we categorize metrics with form $M = \frac{c_i}{c_i + c_j}$, as \emph{Joint Ratio Metrics} (JRMs), which include common measures like true positive rate (TPR), false positive rate (FPR), precision (PPV), and others.

We also consider other metrics, including the new \emph{Marginal Benefit}, which quantifies the net benefit or cost to a group \citep{briscoe2024facets}. Section~\ref{sec:metricDef} details all 21 metrics considered.

\section{MATCH TEST}\label{sect:matchTest}
The inherent variability and jaggedness in classification metrics complicates the assessment of model performance across groups of different sizes. Small variations in the sample size $n$ can lead to disproportionately large changes in metric values, making direct comparisons potentially misleading. This issue is particularly acute in fairness evaluations, where comparing metrics across groups (e.g., different demographic groups) is essential.

To address this challenge, we propose a \emph{Metric Alignment Trial for Checking Homogeneity (MATCH) Test} that quantifies the likelihood of observing a given metric score under a reference distribution. Specifically, we assess whether an observed metric score, $M(\mathrm{CM}_i)$, computed from a confusion matrix $\mathrm{CM}_i$ for group $i$, is consistent with what would be expected if this group followed the same performance distribution as a reference group. By comparing the observed score from group $i$ against the cumulative distribution function (CDF) derived from the reference group, we determine the percentile rank of the observed score within this distribution. This method allows us to evaluate the statistical significance of observed metrics, facilitating more informed interpretations when sample size variations could skew direct comparisons.

In this section, we focus on several 
classification metrics: the six Binomial Metrics, Marginal Benefit, and the eight Joint-Ratio Metrics. For each metric, we derive methods to compute the cumulative probability of an observed score $S_{\mathrm{obs}}$ under the assumption that it is drawn from the reference distribution.

To prepare, we now define several terms. Let $k_i\geq 0$ denote the integer count of $c_i$, $p_i$ is the probability of $c_i$, $q_i=1-p_i$ is the probability of not $c_i$, $S_{\mathrm{obs}}$ is the observed score of a metric $M$, and $P(S\leq S_{\mathrm{obs}})\in[0,1]$ is the probability that $S \leq S_{\mathrm{obs}}$. The subscript $i+j$ means $i$ or $j$.

\subsection{Binomial Metrics}
We find distributions for accuracy, prevalence, and predicted positive which have the form $(c_i+c_j)/n$.

\begin{theorem}
    For Binomial Metrics, the cumulative probability is given by the binomial cumulative distribution function (CDF):
    \begin{align}\label{eq:binomCdf}
        P(S \leq S_{\mathrm{obs}}) = \sum_{f=0}^{k_{i+j}} \binom{n}{f}p^fq^{n-f}
    \end{align}
\end{theorem}
\begin{proof}
The total count $k_{i+j}$ corresponds to the number of successes in $n$ trials, each with success probability $p$. Thus, $k_{i+j}$ follows a binomial distribution with parameters $n$ and $p$.
\end{proof}

When $np \geq 5$ and $nq \geq 5$, we can approximate $P(S \leq S_{\mathrm{obs}})$ using the normal distribution due to the Central Limit Theorem \citep{ye2024probability, freund2014mathematical}. This is done in constant time. By definition, the mean $\mu=np_{i+j}$ and standard deviation $\sigma=\sqrt{np_{i+j}q_{i+j}}$.
\begin{align}
    \text{Standardize: } z = \frac{k_{i+j} + 0.5 - \mu}{\sigma}
\end{align}
The $+0.5$ is the continuity correction as we approximate this discrete distribution with a continuous one. Then, we use the standard normal CDF with $z$: $\Phi(z)$. Furthermore, by the Berry-Esseen Theorem, the error is within $\mathcal{O}\left(n^{-1/2}\right)$ \citep{Berry1941TheAO,esseen1942liapounoff}.

For smaller $n$, exact computation using the binomial CDF is feasible in $\mathcal{O}(n^2)$ time. Otherwise, if the approximation criteria do not hold, one can use efficient approximations such as the Lanczos approximation for $\mathcal{O}(n\log n)$ time \citep{lanczosApprox}.

\textbf{Example.} Consider $n$=100, $p$=0.75, and observed score $S_{\mathrm{obs}}$=0.80. Then $k_{i+j}=nS_{\mathrm{obs}}=80$, $\mu=75$, $\sigma\approx4.33$, and $z\approx1.27$. Thus, $P(S \leq 0.80) \approx \Phi(1.27) \approx 0.90$, so the observed score is higher than approximately 90\% of the expected outcomes under the reference distribution. 

\subsection{Marginal Benefit}
The marginal benefit is defined as $\mathcal{B}$=(FP$-$FN)/$n$, and is designed to measure the net benefit or cost to a group \citep{briscoe2024facets}. Unlike Binomial Metrics, $\mathcal{B}$ involves the difference of counts, making its distribution symmetric around zero when $p_{FP} = p_{FN}$. We provide the CDF for $\mathcal{B}$ in Theorem~\ref{thm:bCdf}, and a theorem for approximating it using the normal distribution in Theorem~\ref{thm:bNormal}.

\begin{theorem}\label{thm:bNormal}
    The cumulative probability for $\mathcal{B}$ can be approximated using the normal distribution:
    \begin{align}
        F(S \leq S_{\mathrm{obs}}) \approx \Phi\left( \frac{k_{FP}-k_{FN} - (p_+-p_-)}{\sqrt{(p_++p_-)-(p_+-p_-)^2}} \right)
    \end{align}
\end{theorem}
\begin{proof}
 $\mathcal{B}$ can be expressed as the average of $n$ i.i.d. random variables $X_i$: $\frac1n \sum_{i=1}^{n}X_i$. For
    \begin{align}
        X_i = \begin{cases}
            +1 & \text{with probability }p_+=p_{FP}\\
            -1 & \text{with probability }p_-=p_{FN}\\
            0 & \text{with probability }p_0 = p_{TN+TP}.
        \end{cases}
    \end{align}
    Now we calculate the moments of $\mathcal{B}$.
    \begin{align}
        \mu = \mathbb{E}[X_i]=p_{+}-p_{-}+0\cdot p_0=p_{+}-p_{-}
    \end{align}
    For $\sigma^2$, we use the expanded definition: $\sigma^2$=$\mathbb{E}[X_i^2]$-$\mu^2$. 
    \begin{align}
    \mathbb{E}[X_i^2]=(+1)^2p_{+}+(-1)^2p_{-}+(0)^2p_0\\
    \text{Thus, }\sigma^2=(p_++p_-)-(p_+-p_-)^2.
    \end{align}
    Then we calculate the z-score and use the normal CDF $\Phi$. Note that the continuity correction $\delta$ cancels out: 
  $  z = \left((k_{FP} + \delta) - (k_{FN} + \delta) - \mu\right)/\sigma$.
\end{proof}
The Berry-Esseen Theorem also guarantees that this error is within $\mathcal{O}(n^{-1/2})$ \citep{Berry1941TheAO,esseen1942liapounoff}.
\subsection{Joint Ratio Metrics}
Joint Ratio Metrics (JRMs) are expressed as $c_i$/($c_i$+$c_j$). Computing the cumulative probability for JRMs presents additional challenges due to the ratio of random variables involved. Theorem \ref{thm:jrmDist} extends the work of \citeauthor{goutte2005probabilistic} by introducing a generalized framework for these metrics, along with an analysis of computational complexity. For simplicity in the following discussion, we define $k_{i+j}$=$k$ and $p_{i+j}$=$p$.
\begin{theorem}\label{thm:jrmDist}
The cumulative probability for a JRM score $S_{\mathrm{obs}}$ is given by:
    \begin{align}
P(S \leq S_{\mathrm{obs}}) = \sum_{k=1}^n P(c_{i+j}=k) P\left(\frac{k_i}{k} \leq S_{\mathrm{obs}} \bigg| c_{i+j}=k \right)\notag\\
 = \sum_{k=1}^n \left(  \binom{n}{k}p^{k}(1-p)^{n-k}\cdot \sum_{k_i=0} ^ {k_i^{\max}}
    \binom{k}{k_i}\theta^{k_i} (1-\theta)^{k -k_i} \right)
    \end{align}
\end{theorem}
\emph{Proof Sketch.} First, $c_{i+j}$ follows a binomial distribution with parameters $n$ and $p$. Given $c_{i+j}$=$ k$, $c_i$ follows a binomial distribution with parameters $k$ and $\theta$=$\frac{p_i}{p}$. Since $c_i$ $\subseteq$ $ c_{i+j}$, $P(c_{i+j}\vert c_i)$=1. Then, by Bayes' Rule, $\theta=P(c_i|c_{i+j})=1\cdot p_i/p$. See Section~\ref{sect:jrmCdf} for a full proof.

Computing this sum directly has computational complexity $\mathcal{O}(n^2)$. However, we can improve efficiency by approximating the distribution of the JRM using the beta distribution. When applying a Bayesian approach with a uniform prior (equivalent to adding pseudo-counts $\lambda = 1$), the posterior distribution of the JRM is a beta distribution with parameters $k_i$+$\lambda$ and $k_j$+$\lambda$.

Using the beta distribution, the cumulative probability is approximated as
\begin{align}
    P(S \leq S_{\mathrm{obs}}) \approx I_{S_{\mathrm{obs}}}(k_i + \lambda, k - k_i + \lambda),
\end{align}
where $I$ is the regularized incomplete beta function. This improves the time complexity to $\mathcal{O}(n)$ since $I$ integrates once over an independent variable. Beta tables allow constant time if they are available. Otherwise, efficient continued fraction and recursive approximations, or quadrature approximations of integrals may allow further improvements as the integrand is well-behaved.

\section{CROSS-PRIOR SMOOTHING}
We introduce \textit{Cross-Prior Smoothing} (CPS), a model-agnostic correction method designed to enhance the reliability of classification metrics. CPS addresses the inherent variability and instability in classification metrics that arise due to sample size disparities while preserving the consistency of metric scores across different groups.

\subsection{The Method}
CPS leverages information from a reference group's confusion matrix, $CM_{\text{ref}}$, to inform and smooth the metrics derived from a target group's confusion matrix, $CM_i$. This approach is especially pertinent in fairness analysis, where both $CM_i$ and $CM_{\text{ref}}$ are generated by the same underlying algorithm but may represent different demographic groups.

Let CM$_{\text{total}}$ denote the confusion matrix derived from the entire dataset, and $CM_i\subset CM_{\text{total}}$ be the confusion matrix for group $i$. In our experiments, we define the reference confusion matrix, $CM_{\text{ref}}$, as the portion of $CM_{\text{total}}$ excluding the data from group $i$: $CM_{\text{ref}}=CM_{\text{total}}\setminus CM_i$.

\begin{assumption}\label{ass:cps}
The reference confusion matrix $CM_{\text{ref}}$ provides a sufficiently informative prior for correcting the metrics derived from $CM_i$.
\end{assumption}

As we demonstrate by our experiments, this assumption is natural in practice, particularly given that once the parameters are fixed for a given model, the multinomial parameters are simple to estimate.

We define CPS in Algorithm~\ref{alg:cps}. We normalize CM$_{\mathrm{ref}}$ to prevent it from dominating the target group that has small sample sizes. Furthermore, the choice of $\lambda$ is crucial for balancing the contribution of the reference group's data.
\begin{algorithm}[t]
\caption{Cross-Prior Smoothing (CPS)}
\label{alg:cps}
\KwIn{Confusion matrix for group $i$, CM$_i$; Normalized reference confusion matrix, $\hat{CM}_{\mathrm{ref}}$; Smoothing constant $\lambda$}
\KwOut{Smoothed confusion matrix CM$_{smooth}$}
Create the concentrations ($\alpha$) using the prior\;
\ForEach{$c\in CM_i$, $c' \in \hat{CM}_{\mathrm{ref}}$}{
  $\alpha_{c} = c + \lambda \cdot c'$\;
}
\ForEach{$c_{smooth} \in $CM$_{smooth}$}{
Normalize the posterior distribution\;
$c_{smooth} =  \alpha_{c}/\left(\sum_{\forall \alpha} \alpha \right)$\;
Scale back for size-dependent metrics\;
$c_{smooth} \gets c_{smooth} \cdot \vert$CM$_i\vert$;
}
\Return Smoothed confusion matrix CM$_{smooth}$\;
\end{algorithm}
Our algorithm is inspired by the Dirichlet distribution, which is a conjugate prior to the multinomial distribution. However, this approach extends beyond a basic Bayesian estimate with a non-informative prior. By incorporating a reference group's data, the CPS technique provides a more robust estimate, particularly when the sample size $n$ is small. This smoothing method significantly improves the stability and reliability of metric estimates by reducing the variability and jaggedness that arise in small-sample scenarios. 

\subsection{Experiments and Results}\label{sec:cpsResultsMain}
We evaluate CPS across a range of commonly used classification metrics, including the eight JRMs, the six binomial metrics, marginal benefit, MCC, F$_1$ score, and prevalence threshold. Our experiments demonstrate that CPS consistently improves the reliability of these metrics, as evidenced by reductions in mean-squared error (MSE).
For empirical validation, we conduct experiments using two fairness datasets: COMPAS and the Folktable income dataset. As datasets commonly used in the fairness literature, these are powerful examples due to their realistically divergent confusion matrices.

\textbf{COMPAS Dataset \citep{propublica2016compasCompas}}: The COMPAS dataset comprises risk assessments used in the U.S. judicial system, which have been scrutinized for potential bias \citep{angwin2016machineCompas, larson2016howCompas, propublica2016compasCompas}. We utilize the confusion matrices reported in \citeauthor{propublica2016compasCompas} to evaluate the effectiveness of CPS in a real-world fairness context.

\textbf{Folktable Income Dataset \citep{folktables}}:\label{sec:folktablesIncomeData} We also employ the Folktable income dataset, training a random forest classifier to predict whether individuals earn more than \$50,000 per year based on features such as marital status, race, and education. For sample-size-induced bias assessment, we partition the data into eight racial groups and compute the corresponding confusion matrices.

For each metric, we conduct ten subgroup-to-reference-group experiments, where we downsample each group via Monte-Carlo one million times for every subgroup sample size, from 5 to 150. This results in \textbf{1.45 billion samples per metric}. Remarkably, CPS improves the MSE over the baseline in every experiment. Figure~\ref{fig:cpsMse} illustrates the reduction in MSE for MCC, FPR, TPR, and PT across all ten groups, with \textbf{95\% confidence intervals that are nearly indistinguishable}.

\begin{figure}[t]
    \centering
    \includegraphics[width=\linewidth]{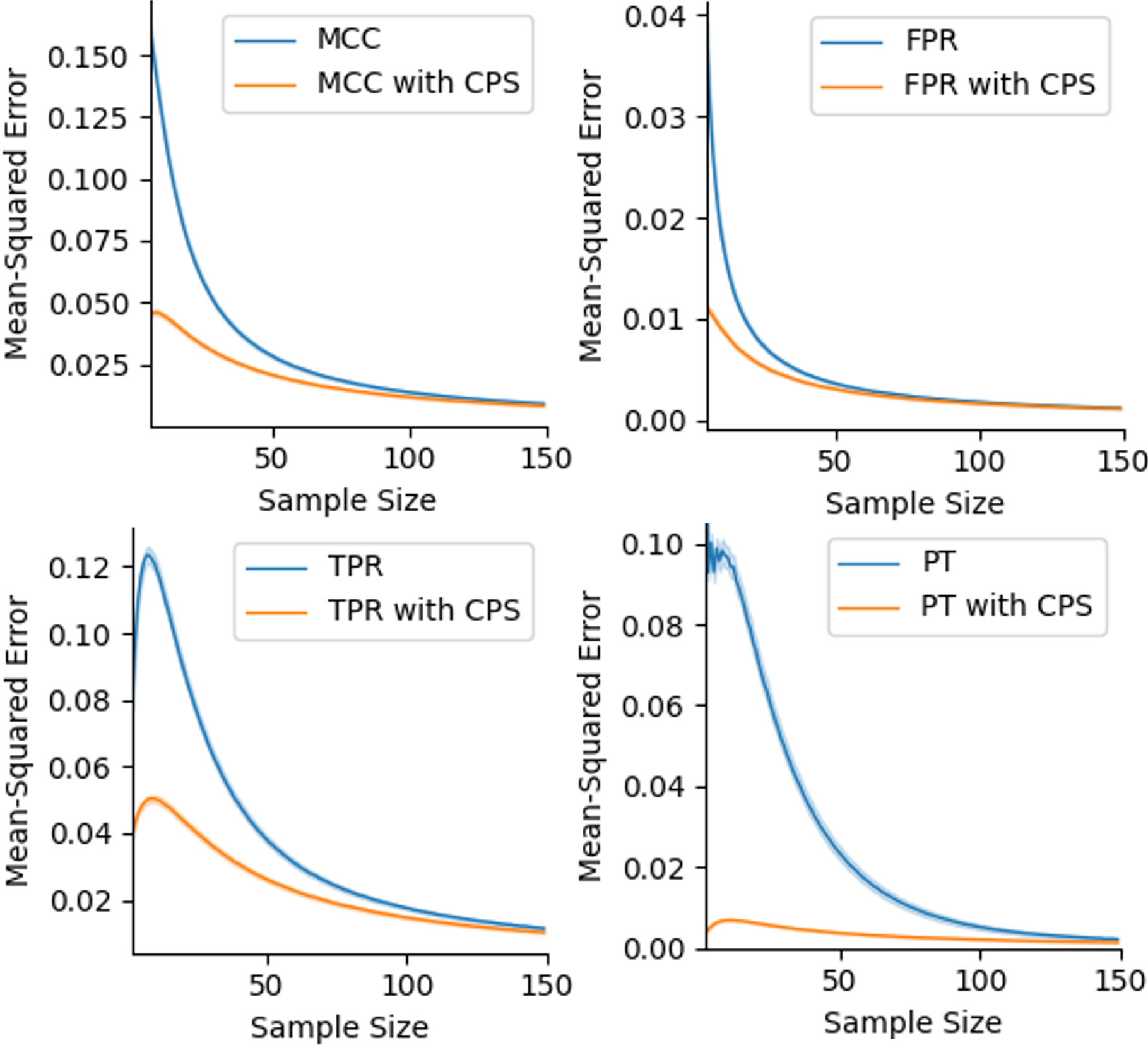}
    \caption{Sample-Size-Induced Bias Among all Groups: Impact of Cross-Prior Smoothing.}
    \label{fig:cpsMse}
\end{figure}

Our improvement in MCC is particularly encouraging, as this metric is well-recognized for its strong informational value \citep{chicco2023matthews}. When ranking metrics by the space between the performance curves, PT shows the highest improvement among the fifteen metrics analyzed, while FPR shows the least. For TPR, we notice a unique behavior, where its error rate rises steadily from n=1 until approximately n=8, before stabilizing. This is because of the small additive smoothing of $1e^{-10}$ applied to the baseline to ensure fair comparisons by minimizing undefined cases with unnoticeable impact on the score. Without the smoothing, TPR follows a monotonic, curved descent, starting at 0.035. 

An additional advantage of CPS is its ability to mitigate issues arising from undefined metric values. By incorporating information from CM$_{\text{ref}}$, CPS drastically reduces the likelihood of encountering zero counts in corresponding cells of CM$_i$, which can lead to undefined scores.

In Figure~\ref{fig:prevCps}, we focus on the prevalence metric to compare the original metric (Laplacian with $\varepsilon=0$), adding one to each cell technique (Laplacian with $\varepsilon=1$), and CPS with $\lambda\in\{5,10,20\}$. We see that CPS improves over all other techniques. We find that increasing CPS's weight ($\lambda$) giving diminishing returns. $\lambda \geq 40$ tends to not perform as well as $\lambda=10$. 

In Section~\ref{sec:cps}, we provide several more experiments. There, we show that the adding one technique is inconsistent, sometimes performing worse than the original metric, and sometimes performing better than CPS. For all fifteen metrics tested with our down-sampling technique, CPS dominates the original metric, being consistently closer to the observed, whole-group score.

\begin{figure}[t]
    \centering
    \includegraphics[width=\linewidth]{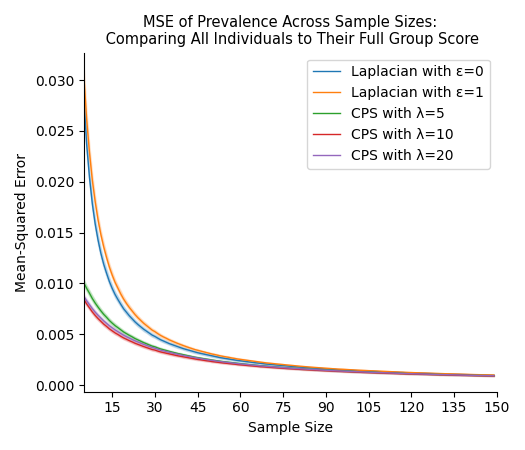}
    \caption{\label{fig:prevCps}Comparing Cross-Prior Smoothing with the original metric, and the add-one technique.}
\end{figure}

\subsection{CPS Gains with Minimal Data: Comparing Prior Assumptions}
\label{subsec:black_cm_priors}

To illustrate the flexibility of CPS in a challenging scenario, we consider a confusion matrix from the COMPAS dataset that reflects only the Black group. This group's confusion matrix is historically cited for its divergence from other racial groups and has invoked great discussion \citep{washington2019lessons,propublica2016compasCompas}. 

Assume we have only a small subset of the Black group data and need a prior to estimate the full Black group's confusion matrix for more reliable metric evaluations.  Table~\ref{tab:compas_cm} shows that while the Black group's confusion matrix does indeed differ from other racial groups, the variance is still small enough that a data-driven prior (CPS) offers a more faithful correction than assuming no prior or a uniform prior (such as the ``add-one'' prior).
By showing that CPS differs the least among all priors, we demonstrate that CPS is applicable even in this controversial case.

\begin{table}[t]
    \centering
    \caption{Evaluating Prior Assumptions for COMPAS' Black Group When Data is Sparse or Unavailable}
    \label{tab:compas_cm}
    \rowcolors{2}{lightgray}{white}
    \renewcommand{\arraystretch}{1.1}
    \begin{tabular}{lcccc}
        \toprule
        \multicolumn{1}{c}{\textbf{Category}} & \textbf{TP} & \textbf{FN} & \textbf{FP} & \textbf{TN} \\
        \midrule
        \textbf{Black (Hidden)} & 0.09 & 0.33 & 0.05 & 0.53 \\
        \textbf{All Other Races} & 0.06 & 0.25 & 0.05 & 0.64 \\
        \midrule
        \multicolumn{5}{c}{\textbf{Possible Priors for the Black Group}} \\
        \midrule
        No Prior Assumption & 0.00 & 0.00 & 0.00 & 0.00 \\
        CPS Estimate & 0.06 & 0.25 & 0.05 & 0.64 \\
        Uniform Prior & 0.25 & 0.25 & 0.25 & 0.25 \\
        \midrule
        \multicolumn{5}{c}{\textbf{Absolute Differences from the Black Group}} \\
        \midrule
        No Prior Assumption & 0.09 & 0.33 & 0.05 & 0.53 \\
        CPS Estimate & \textbf{0.03} & \textbf{0.08} & \textbf{0.00} & \textbf{0.11} \\
        Uniform Prior & 0.16 & \textbf{0.08} & 0.20 & 0.28 \\
        \bottomrule
    \end{tabular}
\end{table}

\subsection{Bias-Variance Trade-Off}
We investigate the bias-variance trade-off of CPS by analyzing the underlying multinomial distribution of CMs. The proof is provided in Section~\ref{sec:biasVar}.
\begin{theorem}\label{thm:biasVar}
    Given the Cross-Prior Smoothing (CPS) estimate $\hat{p}_c^{CPS}$, the bias and variance have the following asymptotic behavior with respect to $n$ and $\lambda$:
    \begin{align}
        \text{Bias}(\hat{p}_c^{CPS}) &\in \mathcal{O}\left(\frac{\lambda}{\lambda + n}\right), \\
        \text{Var}(\hat{p}_c^{CPS}) &\in \mathcal{O}\left(\frac{n}{(n + \lambda)^2}\right).
    \end{align}
\end{theorem}

\begin{corollary}
    Fixing $n$, the bias and variance of $\hat{p}_c^{CPS}$ with respect to $\lambda$ satisfy:
    \begin{align*}
        \text{Bias}(\hat{p}_c^{CPS}) \in \mathcal{O}(1), \quad
        \text{Var}(\hat{p}_c^{CPS}) \in \mathcal{O}\left(\frac{1}{\lambda^2}\right).
    \end{align*}
\end{corollary}

For a given $n$, CPS effectively addresses the \emph{sample-size-induced bias} arising from the discrete, multinomial nature of confusion-matrix metrics (see Figure~\ref{fig:pprMultiracial}) by adjusting the smoothing parameter $\lambda$. Specifically, and with a constant $n$, as $\lambda$ increases, the variance of each cell decreases on the order of $\mathcal{O}(1/\lambda^2)$, while the cell estimates are biased toward $p'_c$ on the order of $\mathcal{O}(1)$. This behavior parallels well-known shrinkage approaches such as the James--Stein estimator, wherein one trades off smaller variance for additional bias. It is important to note, however, that this small-sample ``bias''—an artifact of limited data and discrete counting—does \emph{not} correspond to the classical bias term in the bias-variance trade-off for confusion-matrix cells. Instead, it reflects how some values of $n$ can systematically shift metric scores away from their population-level score. 

\subsection{Discussion}
CPS offers a practical and effective solution for correcting classification metrics in the presence of sample-size-induced bias. By using the reference group's confusion matrix as a prior, we can achieve more stable and reliable metric estimates without compromising the integrity of bias assessments. However, it is important to note that the success of CPS hinges on the validity of Assumption~\ref{ass:cps}. By experimenting with popular bias datasets, we provide supporting evidence of CPS's broad applicability. Nevertheless, we recommend following this practical guide for CPS:
\begin{enumerate}
    \item \textbf{Reasonable Prior Assumption}: CPS is most effective when the reference confusion matrix (the prior) is reasonably expected to be more informative than a uniform prior (e.g., adding one to each cell). Domain knowledge can often justify this; for instance, in the COMPAS dataset, we know that most criminals will not reoffend.
    \item \textbf{Reference Size}: The reference group should have $n\geq 100$. This is due to the reference having sample-sized-induced bias itself. Our experiments show there is significant bias for $n<50$, and the bias becomes negligible once $n \approx 100$.
    \item \textbf{Choosing $\lambda$}: Values in $[5, 10, 20]$ worked well empirically. Users can tune $\lambda$ by iterative refinement.
\end{enumerate}

Alternatively, methods such as synthetic stress tests (down-sampling and cross-validation), marginal distribution comparisons, and divergence measures (e.g., KL divergence) can support CPS' applicability to the given problem.

To evaluate CPS under divergent conditions, we analyze one of the most problematic CMs in the COMPAS dataset, whose systematic bias was publicized by ProPublica and has become a staple in the bias/fairness problems of machine learning. 

\section{CONCLUSIONS}
In this paper, we present a comprehensive analysis of sample-size-induced bias in confusion-matrix metrics, highlighting significant implications for evaluating classification performance and fairness across groups of varying sizes. Our theoretical exploration revealed that small sample sizes lead to increased variability and jaggedness in metric scores due to the discrete and combinatorial nature of confusion matrices.

To address these challenges, we propose two novel approaches. First, we introduce the \emph{MATCH Test} that assesses the statistical significance of an observed metric score relative to a reference distribution. This test enables more informed comparisons across groups by accounting for the variability introduced by different sample sizes. Second, we develop \emph{Cross-Prior Smoothing} (CPS), a method that leverages prior information from a reference group's confusion matrix to correct and stabilize metrics derived from smaller groups. Empirically, we show that CPS improves metric reliability by reducing the error rates for all 15 metrics tested.

Our findings underscore the importance of accounting for sample-size-induced bias when interpreting classification metrics, especially in fairness assessments where group comparisons are critical.

Future work could explore optimizing sample sizes to minimize metric variability and potential exploitation. Investigating the most advantageous configurations that could be misused and refining the MATCH Test by incorporating uncertainty in reference probabilities are also promising directions. Another avenue may explore behavior for different metrics; we have noticed that binomial metrics exhibit the most jaggedness. Additionally, extending the theoretical foundations of CPS and examining its applicability to a broader range of metrics and settings (such as multiclass classification) would further enhance its utility.

\paragraph{Acknowledgments}
This research was supported by the Department of Education’s Graduate Assistance in Areas of National Need (GAANN) Award Number P200A210084. 
\newpage
\bibliographystyle{apalike}
\bibliography{main}
\section*{Checklist}
 \begin{enumerate}
 \item For all models and algorithms presented, check if you include:
 \begin{enumerate}
   \item A clear description of the mathematical setting, assumptions, algorithm, and/or model. [Yes]
   \item An analysis of the properties and complexity (time, space, sample size) of any algorithm. [Yes]
   \item (Optional) Source code, with specification of all dependencies, including external libraries. [\href{https://github.com/jarrenbr/Algorithmic-Accountability-in-Small-Data}{Yes}]
 \end{enumerate}

 \item For any theoretical claim, check if you include:
 \begin{enumerate}
   \item Statements of the full set of assumptions of all theoretical results. [Yes]
   \item Complete proofs of all theoretical results. [Yes--Some are relegated to the supplementary material]
   \item Clear explanations of any assumptions. [Yes]     
 \end{enumerate}

 \item For all figures and tables that present empirical results, check if you include:
 \begin{enumerate}
   \item The code, data, and instructions needed to reproduce the main experimental results (either in the supplemental material or as a URL). [Yes]
   \item All the training details (e.g., data splits, hyperparameters, how they were chosen). [Yes]
         \item A clear definition of the specific measure or statistics and error bars (e.g., with respect to the random seed after running experiments multiple times). [Yes]
         \item A description of the computing infrastructure used. (e.g., type of GPUs, internal cluster, or cloud provider). [Not Applicable]
 \end{enumerate}

 \item If you are using existing assets (e.g., code, data, models) or curating/releasing new assets, check if you include:
 \begin{enumerate}
   \item Citations of the creator If your work uses existing assets. [Not Applicable]
   \item The license information of the assets, if applicable. [Not Applicable]
   \item New assets either in the supplemental material or as a URL, if applicable. [Not Applicable]
   \item Information about consent from data providers/curators. [Not Applicable]
   \item Discussion of sensible content if applicable, e.g., personally identifiable information or offensive content. [Not Applicable]
 \end{enumerate}

 \item If you used crowdsourcing or conducted research with human subjects, check if you include:
 \begin{enumerate}
   \item The full text of instructions given to participants and screenshots. [Not Applicable]
   \item Descriptions of potential participant risks, with links to Institutional Review Board (IRB) approvals if applicable. [Not Applicable]
   \item The estimated hourly wage paid to participants and the total amount spent on participant compensation. [Not Applicable]
 \end{enumerate}

 \end{enumerate}
\appendix
\newpage
\section{Counting Confusion Matrices}\label{sec:cmCount}
Here, we count the number of configurations (unique confusion matrices) possible given $n$ in Theorem~\ref{thm:totalOrdinalCount} and the count of how many times cell $c$ equals some value $x$ given $n$ in Theorem~\ref{thm:ordinalCount}. 

\begin{theorem}\label{thm:totalOrdinalCount}
    The cardinality of $\mathcal{M}(n)$ (the space of all possible confusion matrices given size $n$) is $\mathcal{N}(n) = \binom{n+3}{3} = (n+1)(n+2)(n+3)/6$.
\end{theorem}
\begin{proof}
We prove this using the stars and bars method, which is defined as:
\begin{align}\label{eq:starsBars}
    \binom{n^* + k - 1}{k - 1},
\end{align}
where $n^*$ represents the number of samples, and $k$ denotes the number of buckets. Please see \citet{feller1968} for a review of the stars and bars method.

For our case, we have $n^* = n$ samples and $k = 4$ buckets, which corresponds to $k - 1 = 3$ bars. Applying the formula, we obtain:
\begin{multline}\label{eq:N}
    \vert\mathcal{M}(n)\vert = \mathcal{N}(n) \\= \binom{n+4-1}{4-1} = \frac{(n+1)(n+2)(n+3)}{6}.
\end{multline}

\end{proof}
\begin{corollary}\label{cor:NComplexity}
    From Equation~\ref{eq:N}, the complexity of $\mathcal{N}(n)$ is $\mathcal{O}(n^3)$, and more precisely, $\Theta(n^3)$.
\end{corollary}

Similarly, we can count the number of times a cell $c$ equals some value $x$ given $n$ in Theorem~\ref{thm:ordinalCount}.
\begin{theorem}\label{thm:ordinalCount}
    $C(x;n) = \binom{n-x+2}{2}$, where $x \in \{0, 1, \ldots, n\}$.
\end{theorem}
\begin{proof}
    We use the stars and bars method to count the number of ways to distribute $n-x$ samples into three buckets (CM cells).
\end{proof}

\section{Metric Definitions} \label{sec:metricDef}

In this section, we present a detailed overview of the classification metrics used throughout the paper. These metrics are categorized into three groups: Binomial Metrics, Joint Ratio Metrics (JRM), and Other Metrics, including fairness and specialized measures. Each table provides the formula, abbreviation, and a brief description of the metric to clarify their definitions and usage. We use the notation from Definition~\ref{def:cm}.

\subsection*{Binomial Metrics}

Binomial metrics, such as Accuracy and Prevalence, are fundamental to assessing the overall performance of a classification model. These metrics are based on binomial distributions and provide general insights into model behavior across various datasets. Refer to Table~\ref{tab:Binomial} for a complete breakdown of these metrics.

\subsection*{Joint Ratio Metrics (JRM)}

The Joint Ratio Metrics (JRM) are metrics derived from the confusion matrix, providing insight into the relationships between different types of classification errors. These metrics include widely used measures such as True Positive Rate (TPR), False Positive Rate (FPR), and Positive Predictive Value (PPV). A comprehensive list of JRMs, their formulas, and descriptions can be found in Table~\ref{tab:JRM}.

\subsection*{Other Metrics}

In addition to the JRMs and Binomial Metrics, we also include specialized metrics such as \citeauthor{briscoe2024facets}'s \citeyearpar{briscoe2024facets} Objective Fairness Index, Matthews Correlation Coefficient \citep{matthews1975comparison}, and Prevalence Threshold. These metrics extend beyond basic performance evaluation and focus on more complex aspects, including fairness and correlation in imbalanced datasets. See Table~\ref{tab:OtherMetrics} for further details.

\begin{table*}[htbp]
    \centering
    \caption{Binomial Metrics: General performance metrics using a binomial distribution of outcomes, commonly used for basic model evaluation.}
    \label{tab:Binomial}
    \renewcommand{\arraystretch}{1.5}
    \rowcolors{2}{lightgray}{white}
    \begin{tabular}{L{3.9cm} C{2.1cm} C{2.3cm} L{7cm}}
        \toprule
        \multicolumn{1}{c}{\textbf{Metric}} & 
        \textbf{Abbreviation} & 
        \multicolumn{1}{c}{\textbf{Formula}} & 
        \multicolumn{1}{c}{\textbf{Description}} \\
        \midrule
        \addlinespace
        Binomial Metric & N/A & $ \dfrac{c_i + c_j}{n}$ & A ratio of cells $c_i$ and $c_j$ to the total count $n$.\\
        \addlinespace
        Accuracy & ACC & $ \dfrac{\mathrm{TP} + \mathrm{TN}}{n} $ & Proportion of correct predictions to all predictions. \\
        \addlinespace
        Prevalence & PREV & $ \dfrac{\mathrm{TP} + \mathrm{FN}}{n} $ & Proportion of actual positive instances in the population. (Synonym: $P$) \\
        Predicted Positive Rate & PPR & $ \dfrac{\mathrm{TP} + \mathrm{FP}}{n} $ & Fraction of instances predicted as positive by the classifier. (Synonyms: $\hat{P}$, $PP$) \\
        Inaccuracy & INACC & $\dfrac{\mathrm{FP} + \mathrm{FN}}{n}$ & Fraction of incorrect predictions out of total instances. (Synonyms: Error Rate, Misclassification Rate)\\
        Negative Prevalence & NPREV & $ \dfrac{\mathrm{TN} + \mathrm{FP}}{n} $ & Proportion of actual negative instances in the population. (Synonym: Complement of Prevalence, $N$) \\
        Predicted Negative Rate & PNR & $ \dfrac{\mathrm{TN} + \mathrm{FN}}{n} $ & Proportion of predicted negative instances in the population. (Synonyms: $\hat{N}$, $PN$)  \\
        \bottomrule
    \end{tabular}
\end{table*}

\begin{table*}[htbp]
    \centering
    \caption{Joint Ratio Metrics (JRM): Metrics that describe the relationships between confusion matrix components as ratios, commonly used to assess various types of classification performance errors.}
    \label{tab:JRM}
    \renewcommand{\arraystretch}{1.5}
    \rowcolors{2}{lightgray}{white}
    \begin{tabular}{L{4cm} C{1cm} C{2.9cm} L{7.5cm}}
        \toprule
        \multicolumn{1}{c}{\textbf{Metric}} & 
        \textbf{Abbrev.} & 
        \multicolumn{1}{c}{\textbf{Formula}} & 
        \multicolumn{1}{c}{\textbf{Description}} \\
        \midrule
        Joint Ratio Metric & JRM & $ \dfrac{c_i}{c_i + c_j} $ & A ratio between a cell $c_i$ and the sum of $c_i$ and another cell $c_j$.\\
        True Positive Rate & TPR & $ \dfrac{\mathrm{TP}}{\mathrm{TP} + \mathrm{FN}} $ & Proportion of actual positives correctly identified. (Synonyms: Sensitivity, Recall, Hit Rate) \\
        False Positive Rate & FPR & $ \dfrac{\mathrm{FP}}{\mathrm{FP} + \mathrm{TN}} $ & Proportion of actual negatives incorrectly classified as positive. (Synonym: Fall-Out) \\
        True Negative Rate & TNR & $ \dfrac{\mathrm{TN}}{\mathrm{TN} + \mathrm{FP}} $ & Proportion of actual negatives correctly identified. (Synonym: Specificity) \\
        False Negative Rate & FNR & $ \dfrac{\mathrm{FN}}{\mathrm{FN} + \mathrm{TP}} $ & Proportion of actual positives incorrectly classified as negative. (Synonym: Miss Rate) \\
        Positive Predictive Value & PPV & $ \dfrac{\mathrm{TP}}{\mathrm{TP} + \mathrm{FP}} $ & Proportion of predicted positives that are actual positives. (Synonym: Precision) \\
        Negative Predictive Value & NPV & $ \dfrac{\mathrm{TN}}{\mathrm{TN} + \mathrm{FN}} $ & Proportion of predicted negatives that are actual negatives. \\
        False Discovery Rate & FDR & $ \dfrac{\mathrm{FP}}{\mathrm{FP} + \mathrm{TP}} $ & Proportion of predicted positives that are false positives. \\
        False Omission Rate & FOR & $ \dfrac{\mathrm{FN}}{\mathrm{FN} + \mathrm{TN}} $ & Proportion of predicted negatives that are false negatives. (Synonym: False Reassurance Rate) \\
        \bottomrule
    \end{tabular}
\end{table*}

\begin{table*}[htbp]
    \centering
    \caption{Other Metrics: Specialized metrics focused on fairness, correlation, and balancing different forms of classification errors.}
    \label{tab:OtherMetrics}
    \renewcommand{\arraystretch}{1.2}
    \rowcolors{2}{lightgray}{white}
    \begin{tabular}{L{3.4cm} C{1cm} C{5.25cm} L{5.75cm}}
        \toprule
        \multicolumn{1}{c}{\textbf{Metric}} & 
        \textbf{Abbrev.} & 
        \multicolumn{1}{c}{\textbf{Formula}} & 
        \multicolumn{1}{c}{\textbf{Description}}
        \\
        \midrule
        Original F$_1$ Score & F$_1$ & $2 \div \left(\dfrac{\mathrm{TP} + \mathrm{FP}}{\mathrm{TP}} + \dfrac{\mathrm{TP} + \mathrm{FN}}{\mathrm{TP}}\right)$ & Harmonic mean of precision and recall (simplified). (Synonyms: F$_1$ Measure, Dice-S\o rensen Coefficient)\\
        Simplified F$_1$ Score & F$_1$ & $ \dfrac{2 \cdot \mathrm{TP}}{2 \cdot \mathrm{TP} + \mathrm{FP} + \mathrm{FN}} $ & Simplified version of the Original F$_1$ Score.\\
        Matthews Correlation Coefficient & MCC & \(\frac{\mathrm{TP} \cdot \mathrm{TN} - \mathrm{FP} \cdot \mathrm{FN}}{\sqrt{(\mathrm{TP} + \mathrm{FP})(\mathrm{TP} + \mathrm{FN})(\mathrm{TN} + \mathrm{FP})(\mathrm{TN} + \mathrm{FN})}} \) & Balanced metric accounting for all confusion matrix values, robust for imbalanced data. (Synonyms: Phi Coefficient, $\phi$, $r_\phi$) \\
        Prevalence Threshold & PT & $ \dfrac{\sqrt{\mathrm{TPR}\cdot \mathrm{FPR}}-\mathrm{FPR}}{\mathrm{TPR}-\mathrm{FPR}}$ & Threshold at which the positive prediction rate balances misclassification rates. \\
        Marginal Benefit & $\mathcal{B}$ & $\dfrac{\mathrm{FP} - \mathrm{FN}}{n}$ & Considers objective testing principles in law; represents the benefit gained or lost (cost) for a group.\\
        Objective Fairness Index & OFI & $\mathcal{B}_1 - \mathcal{B}_2$ & The disparity between two groups' marginal benefits. \\
        Treatment Equality & TE & $ \dfrac{\mathrm{FN}_1}{\mathrm{FP}_1} - \dfrac{\mathrm{FN}_2}{\mathrm{FP}_2} $ & Compares false negatives and false positives between two groups. \\
        \bottomrule
    \end{tabular}
\end{table*}

\section{Proofs for Undefined Cases}\label{sec:holes}
In this section, we provide details and proofs for the undefined cases in Table~\ref{tab:holeCounts} from Section~\ref{sec:edgeCases}. See Section~\ref{sec:metricDef} for metric definitions.

\begin{theorem}
    \textbf{Joint Ratio Metrics} have $n+1$ holes for each metric.
\end{theorem}
\begin{proof}
    With form $c_i/(c_i+c_j)$, these metrics are only undefined when $c_i+c_j=0$ (i.e. one row or column sums to zero). For example, consider matrices where the first row sums to zero (i.e. $TP+FN=0$). For these specific matrices, the metrics $TPR=TP/(TP+FN)$ and $FNR = FN/(FN+TP)$ are undefined. This implies that the other entries $TN$ \& $FP$ must sum to $n$. So, picking $TN$ to be any non-negative integer forces $FP=n-TN$. Since we are restricted to choosing $TN=0,1,2,\ldots,n$, there are $n+1$ options for the pair $TN$ and $FP$. As such, there are $n+1$ holes for the metrics $TPR$ and $FNR$. 

    More examples include the cases where the second row sums to zero (holes at $FPR$ \& $TNR$), the first column sums to zero (holes at $PPV$ \& $FDR$), and the second column sums to zero (holes at $NPV$ \& $FOR$). If we assume $c_k$ and $c_\ell$ are the entries in the zero row or column while $c_i$ and $c_j$ are the others, then we can count using the same restriction from above, $c_i=n-c_j$. Since there are $n+1$ options for the pair $c_i$ and $c_j$, there are $n+1$ holes for the metric undefined when $c_k+c_\ell=0$. Consequently, each JRM has $n+1$ matrices at which they will be undefined. 
\end{proof}

\begin{theorem}
The \textbf{original} $\mathbf{F_1}$\textbf{ Score} has $\binom{n+2}{2}$ holes.
\end{theorem}
\begin{proof}This score relies on the inverse of precision and recall being defined. As such, it is undefined when $\text{TP}=0$. By Theorem~\ref{thm:ordinalCount}, some CM cell (e.g., TP) has $\binom{n+2}{2}$ counts of zero over all possible CM configurations given $n$.
\end{proof}

\begin{theorem}
The \textbf{simplified} $\mathbf{F_1}$\textbf{ Score} has one hole.
\end{theorem}
\begin{proof}
Introduced as the harmonic mean of precision and recall, modern ML benchmarks often use the simplified version: 2TP/(2TP+FP+FN). Here, $F_1$ is only undefined when 2TP+FP+FN=0, giving 1 hole. 
\end{proof}

\begin{theorem}
The \textbf{Matthews Correlation Coefficient} has $4n$ holes.
\end{theorem}
\begin{proof}
    Defined as $\frac{TP \times TN - FP \times FN}{\sqrt{(TP + FP)(TP + FN)(TN + FP)(TN + FN)}}$, MCC is only undefined when any term in the denominator's product is zero (e.g., TP+FN=0). From the JRM proof, there are $n$+1 configurations for $c_i$+$c_j$=0. However, there are four overlapping cases, all when some $c_i=n$. Hence, there are $4(n+1)-4=4n$ holes.
\end{proof}

\begin{theorem}
The \textbf{Prevalence Threshold} has at least $2n+2$ holes and less than $e^\gamma n\log \log n + \frac{0.6483n}{\log \log n}$ holes for all $n\geq 3$.

\emph{Note:} Since $e$ is Euler's number and $\gamma$ is the Euler-Mascheroni constant, $e^\gamma \approx 1.78$.
\end{theorem}
\begin{proof}
    Defined as $\frac{\sqrt{TPR\cdot FPR}-FPR}{TPR-FPR}$, PT is undefined if TP+FN=0 (TPR), or FP+TN=0 (FPR), or TPR$-$FPR=0. TPR and FPR are JRM metrics, hence have n+1 holes each. We now move to TPR$-$FPR=TP/(TP+FN)$-$FP/(FP+TN)=0.
\begin{eqnarray}&TP/(TP+FN)-FP/(FP+TN) = 0 \\ \iff &TP(FP + TN) - FP(TP +FN) = 0 \\ \iff &TP\cdot TN - FP\cdot FN = 0\label{eq:ptIffFinal}\end{eqnarray} 
Equation~\ref{eq:ptIffFinal} only occurs if one row or column is a multiple of the other. For a matrix with a zero row or column, this is 
 the set of MCC holes described above. 
 Otherwise, there is at most one 
 multiple of that column such that the sum of the entries is $n$. For example, if the predicted-negative 
 column is an integer ($k$) multiple of the predicted-positive column, then we have $k\cdot$(TP+FN)=$n$ 
 and hence  (TP+FN)$|n$. 
Thus, we can bound the number of holes using the sum of divisor function. Then, by Robin's Theorem \citep{robin1984},
there are strictly less than $e^\gamma n\log \log n + \frac{0.6483n}{\log \log n}$ holes for $n\geq 3$. 

Furthermore, since the total holes from FPR and TPR equals $2n+2$, the lower bound is $\Omega(n)$.
\end{proof}

\begin{theorem}
The \textbf{Treatment Equality} metric has $\binom{n_1+2}{2}+\binom{n_2+2}{2} - 1$ holes.
\end{theorem}
\begin{proof}
    Recall that TE is defined as the difference of the ratios $\frac{FN_1}{FP_1} - \frac{FN_2}{FP_2}$ for two groups within the dataset: 1, 2. The metric is undefined when either $FP_1 = 0$ or $FP_2 = 0$. For each group, the number of ways for $FP_i = 0$ is $\binom{n_i + 2}{2}$, where $n_i$ is the sample size of group $i$. Summing over both groups and subtracting the case where both $FP_1 = FP_2 = 0$ (to avoid double-counting), the total number of holes is $\binom{n_1+2}{2}+\binom{n_2+2}{2} - 1$.
\end{proof}

\section{Marginal Benefit's Distribution}
In this section, we define marginal benefit's ($\mathcal{B}$'s) probability mass function in Theorem~\ref{thm:bPmf} and its cumulative distribution function in Theorem~\ref{thm:bCdf}. Understanding this distribution allows us to compute the exact probability of observing a given value of $\mathcal{B}$ under the reference distribution.

\begin{figure*}[htbp]
\begin{theorem}[PMF of $\mathcal{B}$]\label{thm:bPmf}
    We find the PMF of $\mathcal{B}=\frac{k}{n}$ by summing over the probabilities of all counts FP and FN such that $FP-FN=k$.
    \begin{equation}
        P\left( \mathcal{B} = \dfrac{k}{n} \right) = \sum_{n_{-1} = n_{-1}^{\min}}^{n_{-1}^{\max}} \left( 
        \underbrace{\dfrac{n!}{(k + n_{-1})! \, n_{-1}! \, (n - k - 2 n_{-1})!}}_{\text{Multinomial Coefficient}}
         \ \times \  
         \underbrace{p_{\mathrm{FP}}^{k + n_{-1}} \, p_{\mathrm{FN}}^{n_{-1}} \, p_0^{n - k - 2 n_{-1}}}_{\text{Joint Probability Term}}
         \right)
    \end{equation}
    where: $\quad
n_{-1}^{\min} = \max(0, -k), \quad
n_{-1}^{\max} = \left\lfloor \dfrac{n - k}{2} \right\rfloor, \quad$ and $k$ ranges over the integers $-n$ and $n$ since\\$k=FP-FN \in[-n,n]$.
\end{theorem}
\end{figure*}

\begin{figure*}[htbp]
    \begin{theorem}[CDF of $\mathcal{B}$]\label{thm:bCdf}
    The cumulative distribution function of $\mathcal{B}$ is:
        \begin{equation}
            P\left( \mathcal{B} \leq b \right) = P\left( S \leq n b \right) = \sum_{k = -n}^{\lfloor n b \rfloor} \ 
            \underbrace{\sum_{n_{-1} = n_{-1}^{\min}(k)}^{n_{-1}^{\max}(k)}
            \left(
            \dfrac{n!}{(k + n_{-1})! \, n_{-1}! \, (n - k - 2 n_{-1})!} \times p_{\mathrm{FP}}^{k + n_{-1}} \, p_{\mathrm{FN}}^{n_{-1}} \, p_0^{n - k - 2 n_{-1}}
            \right)
            }_{\text{Theorem~\ref{thm:bPmf}}}
        \end{equation}
        where $b$ is a real number ranging from $-1$ to $1$ and $k$ is an integer ranging from $-n$ to $\lfloor nb \rfloor$.
    \end{theorem}
    \end{figure*}

\begin{proof}[Proof of Theorem~\ref{thm:bPmf}]
    Recall that $\mathcal{B}$ can be expressed as the average of $n$ i.i.d. random variables $X_i$: $\frac1n \sum_{i=1}^{n}X_i$. For
    \begin{align}\label{eq:xiCases}
        X_i = \begin{cases}
            +1 & \text{with probability }p_+=p_{FP}\\
            -1 & \text{with probability }p_-=p_{FN}\\
            0 & \text{with probability }p_0 = p_{TN+TP}.
        \end{cases}
    \end{align}
We proceed with the proof in two parts. First, we determine the bounds, and then we derive the corresponding multinomial probability.

\subsubsection*{The Bounds of the PMF:}

Let $n_{+1}$ be the number of times $X_i = +1$ (number of false positives), $n_{-1}$ be the number of times $X_i = -1$ (number of false negatives), and $n_0$ be the number of times $X_i = 0$ (number of true positives and true negatives). Then:
\begin{align*}
n &= n_{+1} + n_{-1} + n_0 \\
S = k &= n_{+1} \cdot (+1) + n_{-1} \cdot (-1) + n_0 \cdot 0 = n_{+1} - n_{-1}
\end{align*}
\begin{equation}\label{eq:k}
\text{We highlight } k = n_{+1} - n_{-1} \text{ from above.}
\end{equation}
It follows that $n_{+1}=k+n_{-1}$. We use this to redefine $n_0$ in terms of $n$, $k$, and $n_{-1}$:
\begin{align}\label{eq:n0}
    n_0 = n-(k+n_{-1})-n_{-1} = n-k-2n_{-1}.
\end{align}

Our bounds are constrained by all cells being non-negative. As such, the counts $n_{+1}$, $n_{-1}$, and $n_0$ must satisfy:
\begin{align*}
n_{+1} &\geq 0 \implies k + n_{-1} \geq 0 \implies n_{-1} \geq -k \\
n_{-1} &\geq 0 \\
n_0 &\geq 0 \implies n - k - 2 n_{-1} \geq 0 \implies n_{-1} \leq \dfrac{n - k}{2}
\end{align*}
Hence, the valid range for $n_{-1}$ is:
\begin{align}
n_{-1}^{\min} = \max(0, -k), \quad
n_{-1}^{\max} = \left\lfloor \dfrac{n - k}{2} \right\rfloor
\end{align}

\subsubsection*{The Multinomial Probability:}

By definition, the multinomial probability of observing counts $(n_{+1}, n_{-1}, n_0)$ is:
\begin{multline}
P(n_{+1}, n_{-1}, n_0) = \\
 \underbrace{\dfrac{n!}{n_{+1}! \, n_{-1}! \, n_0!}}_{\text{Multinomial Coefficient}}
\quad \mathbf{\times} \quad 
\underbrace{p_{\mathrm{FP}}^{n_{+1}} \, p_{\mathrm{FN}}^{n_{-1}} \, p_0^{n_0}}_{\text{Joint Probability Term.}}
\end{multline}

Then by substituting $n_{+1} = k + n_{-1}$ (Equation~\ref{eq:k}) and $n_0$ with Equation~\ref{eq:n0} into $P(S=k)$, we have:
\begin{equation*}
P(S=k) = P\left( \mathcal{B} = \dfrac{k}{n} \right).
\end{equation*}
This concludes the proof of Theorem~\ref{thm:bPmf}'s formula for $\mathcal{B}$'s PMF.
\end{proof}

\begin{proof}[Proof of Theorem~\ref{thm:bCdf}]
    We start by noting that $\mathcal{B} = \dfrac{S}{n}$, where $S = \mathrm{FP} - \mathrm{FN} = \sum_{i=1}^{n} X_i$ is the sum of independent random variables $X_i$ defined in Theorem~\ref{thm:bPmf} (Equation \ref{eq:xiCases}). Furthermore, we reframe $P(\mathcal{B}<b)$ as $P(\mathcal{B} \leq b) = P\left( \dfrac{S}{n} \leq b \right) = P\left( S \leq n b \right)$.

    Since $S=FP-FN$ is an integer ranging from $-n$ to $n$, the possible values of $\mathcal{B}$ are $\dfrac{k}{n}$ for integer $k\in[-n,n]$. Thus, $\mathcal{B}$'s CDF is:
    \begin{align*}
P\left( \mathcal{B} \leq b \right) = \sum_{k = -n}^{\lfloor n b \rfloor} P(S = k).
\end{align*}
\end{proof}
\section{Joint Ratio Metrics' CDF} \label{sect:jrmCdf}
In this section, we provide the full proof for Theorem~\ref{thm:jrmDist} which gives the CDF of Joint Ratio Metrics (JRMs). For simplicity in the following discussion, we define $k_{i+j}=k$ and $p_{i+j}=p$.
 
 \textbf{Theorem~\ref{thm:jrmDist}} (Restated). \emph{
The cumulative probability for a JRM score $S_{\mathrm{obs}}$ is given by:}
\begin{multline}
    P(S \leq S_{\mathrm{obs}}) \\
    = \sum_{k=1}^n P(c_{i+j}=k) P\left(\frac{k_i}{k} \leq S_{\mathrm{obs}} \bigg| c_{i+j}=k \right)\\
     = \sum_{k=1}^n \left(  \binom{n}{k}p^{k}(1-p)^{n-k}\cdot \sum_{k_i=0} ^ {k_i^{\max}}
        \binom{k}{k_i}\theta^{k_i} (1-\theta)^{k -k_i} \right)
\end{multline}
\begin{proof}
From our analysis in binomial metrics, we have $c_{i+j}$ and $c_i$ following binomial distributions:
\begin{align*}
    c_{i+j} \sim Binomial(n,p_{i+j}), \quad c_i \sim Binomial(n,p_i).
\end{align*}
Since $S$ is only defined for $k_i+k_j=k>0$, we focus on $k \geq 1$. We begin with the simplest solution in Equation \ref{eq:simpleJoint} which sums all probabilities for scores lower than $S_{\text{obs}}$.
\begin{align}\label{eq:simpleJoint}
P(S\leq S_\text{obs})
= \sum_{k=1}^{n}\sum_{k_i=0}^n  \begin{cases}
P(k_i \big| c_{i+j}=k) & \text{if }\frac{k_i}{k} \leq S_\text{obs}\\
0 & \text{otherwise.}
\end{cases}
\end{align}

By definition, if $c_i$ is sampled, then $c_{i+j}$ is also sampled, implying that $P(c_{i+j} \mid c_i) = 1$. Using Bayes' rule, we can express:
\begin{align*}
\theta = P(c_i \mid c_{i+j}) = \frac{p_i}{p_{i+j}}.
\end{align*}
Consequently, we model the conditional sampling from a JRM $M$ as:
\begin{align*}
c_i \mid c_{i+j} \sim \text{Binomial}(c_{i+j}, \theta).
\end{align*}
We now proceed to derive the PMFs to construct the CDFs. Since the events are binomially distributed, the PMFs are given by:
\begin{align}
    P(c_{i+j}) &= \binom{n}{c_{i+j}}p_{i+j}^{c_{i+j}}(1-p_{i+j})^{n-c_{i+j}}\\
    P(c_i|c_{i+j}) &= \binom{c_{i+j}}{c_i}\theta^{c_i}(1-\theta)^{c_{i+j}-c_i}
\end{align}
Then, the CDF of $\text{Binomial}(c_{i+j}, \theta)$ can be expressed as:
    \begin{multline}
        P\left(\frac{k_i}{k} \leq S_{obs} \bigg| c_{i+j} = k\right)\\
        = P\left(k_i \leq S_{obs}\cdot k \bigg| c_{i+j}=k\right)\\
        = \sum_{k_i=0} ^ {k_i^{\max}}
        \binom{k}{k_i}\theta^{k_i} (1-\theta)^{k -k_i}
    \end{multline}
where $k_i^{\max} = \lfloor S_{obs}\cdot c_{i+j}\rfloor$. 

By expanding the PMF of $c_{i+j}$ and combining it with the CDF of $c_i \mid c_{i+j}$, we prove Theorem~\ref{thm:jrmDist}.
\end{proof}

\section{Cross-Prior Smoothing}
We expand upon CPS with experiments and the proof for the bias-variance tradeoff.
\subsection{Additional Experiments}
\label{sec:cps}
We conduct extensive experiments across fifteen commonly used classification metrics, including the eight JRMs, accuracy, prevalence, predicted positive rate, marginal benefit, MCC, F$_1$ score, and prevalence threshold. The experiments are performed on two popular fairness datasets: the COMPAS dataset and the Folktable Income dataset. For each metric, we test 1.45 billion samples over the datasets, as described in Section~\ref{sec:cpsResultsMain}.

We compare the performance of the original metrics, metrics with additive smoothing (``bashful'' smoothing with $\varepsilon=1\times10^{-10}$ and the non-informative prior $\varepsilon=1$), and metrics smoothed with CPS and the priori scales $\lambda=5$, $\lambda=10$, and $\lambda=20$.
\subsection{Results and Analyses}
Our results show that CPS consistently improves the reliability of the classification metrics, particularly for small sample sizes. Figures~\ref{fig:mccAndTprMore} and~\ref{fig:acc_and_b_mse} illustrate the Mean-Squared Error (MSE) of various metrics as a function of the sample size, comparing the original metrics, metrics with additive smoothing, and CPS with different $\lambda$ values.

Our experimental analyses excluded undefined values, which could make direct comparisons between CPS and the original metrics (which often include such undefined cases) potentially unfair. Nevertheless, it is important to maintain scores that closely resemble the original metrics. Therefore, we set $\varepsilon=1e-10$ in Section~\ref{sec:cpsResultsMain}. By introducing CPS, we improve over the original metrics when undefined values arise, as illustrated in Figure~\ref{fig:mccAndTprMore}. Furthermore, Figure~\ref{fig:mccAndTprMore} also demonstrates that CPS performance improves when increasing $\lambda$ from 5 to 20.

However, in the case of metrics that are well defined even with zero counts (e.g., binomial metrics and marginal benefit), bashful smoothing has minimal impact on error rates, as shown in Figure~\ref{fig:acc_and_b_mse}. CPS similarly improves reliability in these metrics.

\begin{figure*}[htbp]
    \centering
    \begin{subfigure}[b]{0.495\linewidth}
        \centering
        \includegraphics[width=\linewidth]{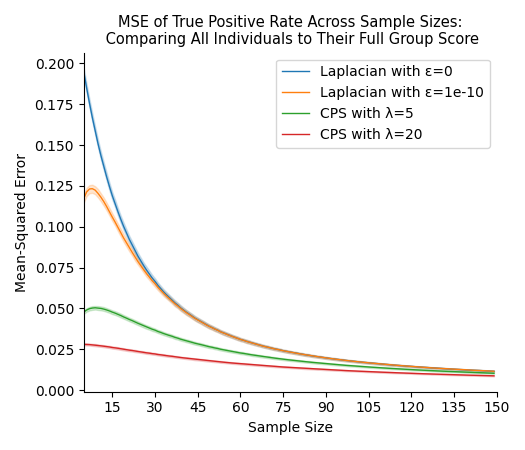}
        \label{fig:tprMore}
    \end{subfigure}
    \hfill
    \begin{subfigure}[b]{0.495\linewidth}
        \centering
        \includegraphics[width=\linewidth]{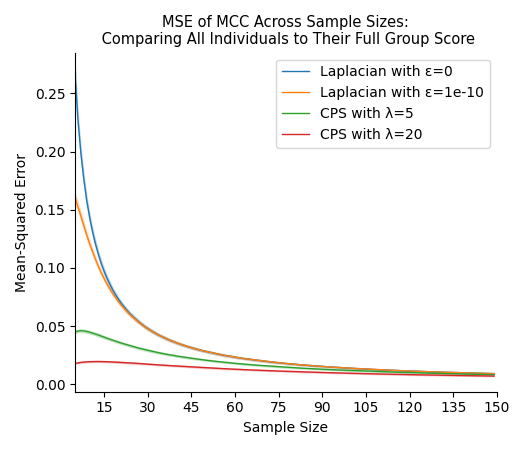}
        \label{fig:mccMore}
    \end{subfigure}
    \caption{Effect of smoothing techniques on metrics with undefined values. Bashful smoothing ($\varepsilon=1e^{-10}$) reduces errors compared to no smoothing ($\varepsilon=0$). Cross-Prior Smoothing (CPS) offers further improvements, with the stronger prior ($\lambda=20$) outperforming the weaker prior ($\lambda=5$), indicating that CPS uses sufficiently informative priors.}
    \label{fig:mccAndTprMore}
\end{figure*}
\begin{figure*}[htbp]
    \centering
    \begin{subfigure}[b]{0.495\linewidth}
        \centering
        \includegraphics[width=\linewidth]{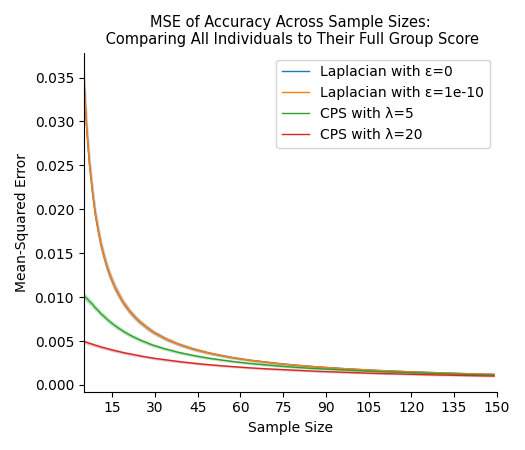}
        \label{fig:accMore}
    \end{subfigure}
    \hfill
    \begin{subfigure}[b]{0.495\linewidth}
        \centering
        \includegraphics[width=\linewidth]{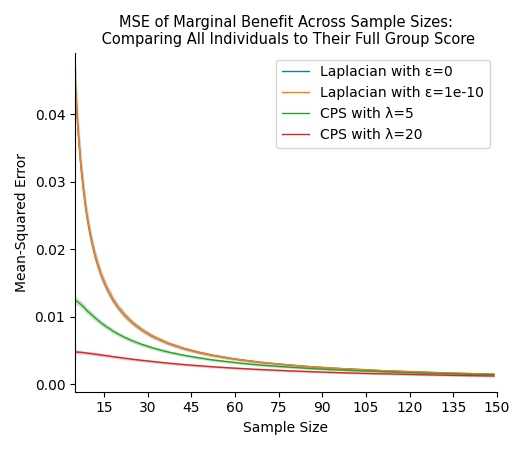}
        \label{fig:prevMore}
    \end{subfigure}
    \caption{Comparison of smoothing effects on metrics without holes. Results show that applying a small smoothing factor ($\varepsilon=1e^{-10}$) has minimal impact compared to no smoothing. CPS continues to reduce errors, following the same trend observed in Figure~\ref{fig:mccAndTprMore}.}
\label{fig:acc_and_b_mse}
\end{figure*}
We also explore the effect of the most common non-informative prior, $\varepsilon=1$. While this approach reduces errors for some metrics, it leads to inconsistent and less predictable results as shown in Figure~\ref{fig:prevAndFnr}. In some cases, smoothing with $\varepsilon=1$ performs worse than no smoothing (e.g., NPV, FOR, and prevalence). In other cases, it performs better than CPS with $\lambda=10$ for certain metrics (e.g., FNR and TPR). However, CPS with higher $\lambda$ values still outperforms additive smoothing with $\varepsilon=1$ overall.
\begin{figure*}[htbp]
    \centering
    \begin{subfigure}[b]{0.495\linewidth}
        \centering
        \includegraphics[width=\linewidth]{img/All_PREV_mse.png}
        \label{fig:prevWorse}
    \end{subfigure}
    \hfill
    \begin{subfigure}[b]{0.495\linewidth}
        \centering
        \includegraphics[width=\linewidth]{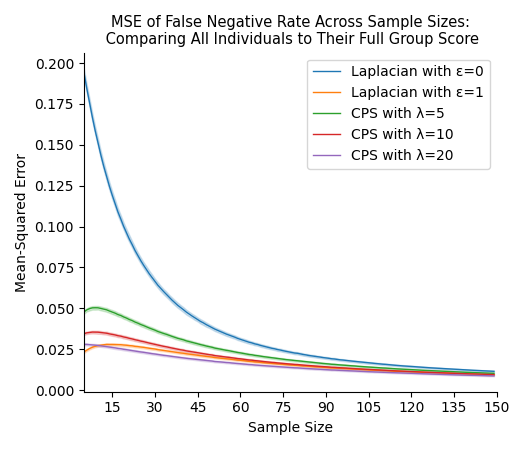}\label{fig:fnrBetter}
    \end{subfigure}
    \caption{Smoothing with $\varepsilon=1$ yields inconsistent results. In some metrics (left), it performs worse than no smoothing, while in others (right), it outperforms CPS with $\lambda=10$.}
\label{fig:prevAndFnr}
\end{figure*}
Our experiments confirm that the error bands (95\% confidence intervals) in the plots are narrow and often indistinguishable, indicating high consistency across datasets and groups. The large number of samples (one million) ensures that the results are statistically significant and representative.

\subsection{Bias and Variance}\label{sec:biasVar}

\textbf{Theorem~\ref{thm:biasVar}} (Restated). \emph{The bias and variance of the CPS estimator are $\mathcal{O}\left(\frac{\lambda}{\lambda + n}\right)$ and $\mathcal{O}\left(\frac{n}{(n + \lambda)^2}\right)$, respectively.}

\begin{proof}[Proof of Theorem~\ref{thm:biasVar}]
Consider a cell with true probability $p_c$ and reference probability $p'_c$.

\begin{itemize}
    \item Without smoothing, $\hat{p}_c = C / n$.
    \item With CPS, $\hat{p}_c^{CPS} = (C + \lambda p'_c) / (n + \lambda)$.
\end{itemize}

Substituting $\mathbb{E}[C] = np_c$ yields
\begin{equation*} 
    \mathbb{E}[\hat{p}_c^{CPS}] = (np_c + \lambda p'_c) / (n + \lambda),
\end{equation*}
so the bias is:
\[
\text{Bias}(\hat{p}_c^{CPS}) = \mathbb{E}[\hat{p}_c^{CPS}] - p_c = \frac{\lambda}{n + \lambda} (p'_c - p_c).
\]

As a binomial distribution, $\text{Var}(C) = np_c (1 - p_c)$. Since $\hat{p}_c^{CPS}$ is a sample proportion normalized by $n + \lambda$:
\[
\text{Var}(\hat{p}_c^{CPS}) = \frac{\text{Var}(C)}{(n + \lambda)^2} = \frac{np_c (1 - p_c)}{(n + \lambda)^2}.
\]

In summary,
\begin{align}
    \text{Bias}(\hat{p}_c^{CPS}) & = \frac{\lambda}{n + \lambda} (p'_c - p_c), \\
    \text{Var}(\hat{p}_c^{CPS}) & = \frac{np_c (1 - p_c)}{(n + \lambda)^2}.
\end{align}
\end{proof}

\section{Reproducing Results}
Anyone can reproduce our results by following these steps:

\begin{enumerate}
    \item Download the \href{https://github.com/jarrenbr/Algorithmic-Accountability-in-Small-Data}{repository}.
    \item Create a Python 3.11 virtual environment.
    \item Install the required dependencies:
    \begin{lstlisting}[language=bash]
    pip install -r requirements.txt
    \end{lstlisting}

    \item Run the main script:
    \begin{lstlisting}[language=bash]
    python main.py
    \end{lstlisting}
\end{enumerate}

\end{document}